\documentclass{sn-jnl}
\usepackage{multirow}%
\usepackage{mathrsfs}%
\usepackage{textcomp}%
\usepackage{manyfoot}%
\usepackage{listings}%
\usepackage{times}
\usepackage{soul}
\usepackage{url}
\usepackage{amsmath}
\usepackage{amsthm}
\usepackage{booktabs}
\usepackage{algorithm}
\usepackage{algorithmic}
\usepackage{unicode}
\usepackage[T1]{fontenc}
\usepackage[utf8]{inputenc}
\usepackage{subcaption}
\usepackage{booktabs}
\usepackage[table,xcdraw]{xcolor}
\usepackage{longtable}
\usepackage{array}
\usepackage{thmtools}
\usepackage{natbib}
\usepackage{cleveref}
\newtheorem{theorem}{Theorem}%  meant for continuous numbers

\newtheorem{definition}{Definition}%

\begin{document}

\title[Article Title]{Reduced Implication-bias Logic Loss for Neuro-Symbolic Learning}

\author[1]{\fnm{Hao-Yuan} \sur{He}\textsuperscript{\href{https://orcid.org/0000-0002-8433-4782}{\includegraphics[scale=0.02]{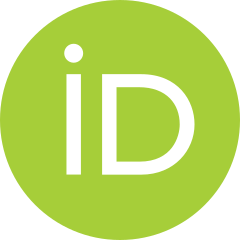}}}}\email{hehy@lamda.nju.edu.cn}

\author[1]{\fnm{Wang-Zhou} \sur{Dai}}\email{daiwz@lamda.nju.edu.cn}

\author*[1]{\fnm{Ming} \sur{Li}}\email{lim@lamda.nju.edu.cn}

\affil*[1]{\orgname{Nanjing University}, \orgaddress{\city{Nanjing}, \postcode{210023}, \country{China}}}

\abstract{
    Integrating logical reasoning and machine learning by approximating logical inference with differentiable operators is a widely used technique in the field of Neuro-Symbolic Learning.
    However, some differentiable operators could introduce significant biases during backpropagation, which can degrade the performance of Neuro-Symbolic systems.
    In this paper, we demonstrate that the loss functions derived from fuzzy logic operators commonly exhibit a bias, referred to as \emph{Implication Bias}.
    To mitigate this bias, we propose a simple yet efficient method to transform the biased loss functions into \emph{Reduced Implication-bias Logic Loss (RILL)}.
    Empirical studies demonstrate that RILL outperforms the biased logic loss functions, especially when the knowledge base is incomplete or the supervised training data is insufficient.
}

\keywords{Implication Bias, Neuro-Symbolic Learning, Neural Networks, Machine Learning}

\maketitle

\section{Introduction}

Neuro-Symbolic (NeSy) AI~\citep{garcez2019neural,de2020statistical,yang_ste_2022} aims to bridge the gap between neural networks and symbolic reasoning to achieve a more comprehensive form of artificial intelligence. Some researchers attempt to create a hybrid system by developing an interface between the neural and symbolic modules. For instance, \citet{dai_abl_2019} and \citet{zhou2019abductive} introduce the Abductive Learning (ABL) framework, which combines first-order logic with machine learning models and abductive reasoning. Additionally, \citet{manhaeve_deepproblog_2018} propose DeepProbLog, which integrates probabilistic logic programming with deep learning through neural predicates.

Training hybrid models can be challenging due to the complexity of jointly optimizing the neural and symbolic modules. 
To address this challenge, some researchers have proposed approximating logical reasoning with differentiable operators~(e.g., fuzzy operators) in order to transform symbolic knowledge into loss functions~\citep{xu_semantic_2018,giannini2019relation,van_krieken_survey_2022}. 
Models can then be trained using gradient descent to improve the efficiency of the training process.

Despite the benefits in efficiency, there might remain some problems when approximating the discrete logical calculations naively~\citep{van_krieken_survey_2022}.
When approximating implication rules such as $p_1\land p_2\land\cdots\land p_n \rightarrow q$ with differentiable operations, a phenomenon that we called \textit{Implication Bias} could degrade the model performance. Informally, this rule could be satisfied via vacuous truth~\citep{DBLP:books/daglib/0076838}, i.e., by negating the premises $p_1\land p_2\land\cdots\land p_n$.
As shown in \cref{fig:ib}, NeSy systems can increase the consistency between the predictions and their knowledge base by negating the premises of implication rules.

Suppose the NeSy system is trained under optimal conditions, such as having access to a complete knowledge base that can detect these shortcut errors or having sufficient supervised training data that can teach the model right from wrong. 
In that case, the system can fortunately avoid introducing biases during training.
However, real-world problems do not always have idealistic conditions such as a complete knowledge base or sufficient supervised data. 
In such cases, implication bias can have a significant negative impact, leading to shortcuts in the training process. 
Additionally, \citet{DBLP:journals/natmi/GeirhosJMZBBW20} note that neural networks have a tendency to fit these shortcuts during training. 
As a result, the system may prioritize satisfying the knowledge base rather than making accurate predictions, leading to suboptimal outcomes.

% 说明方法
\begin{figure}
    \centering
    \includegraphics[width=0.8\linewidth]{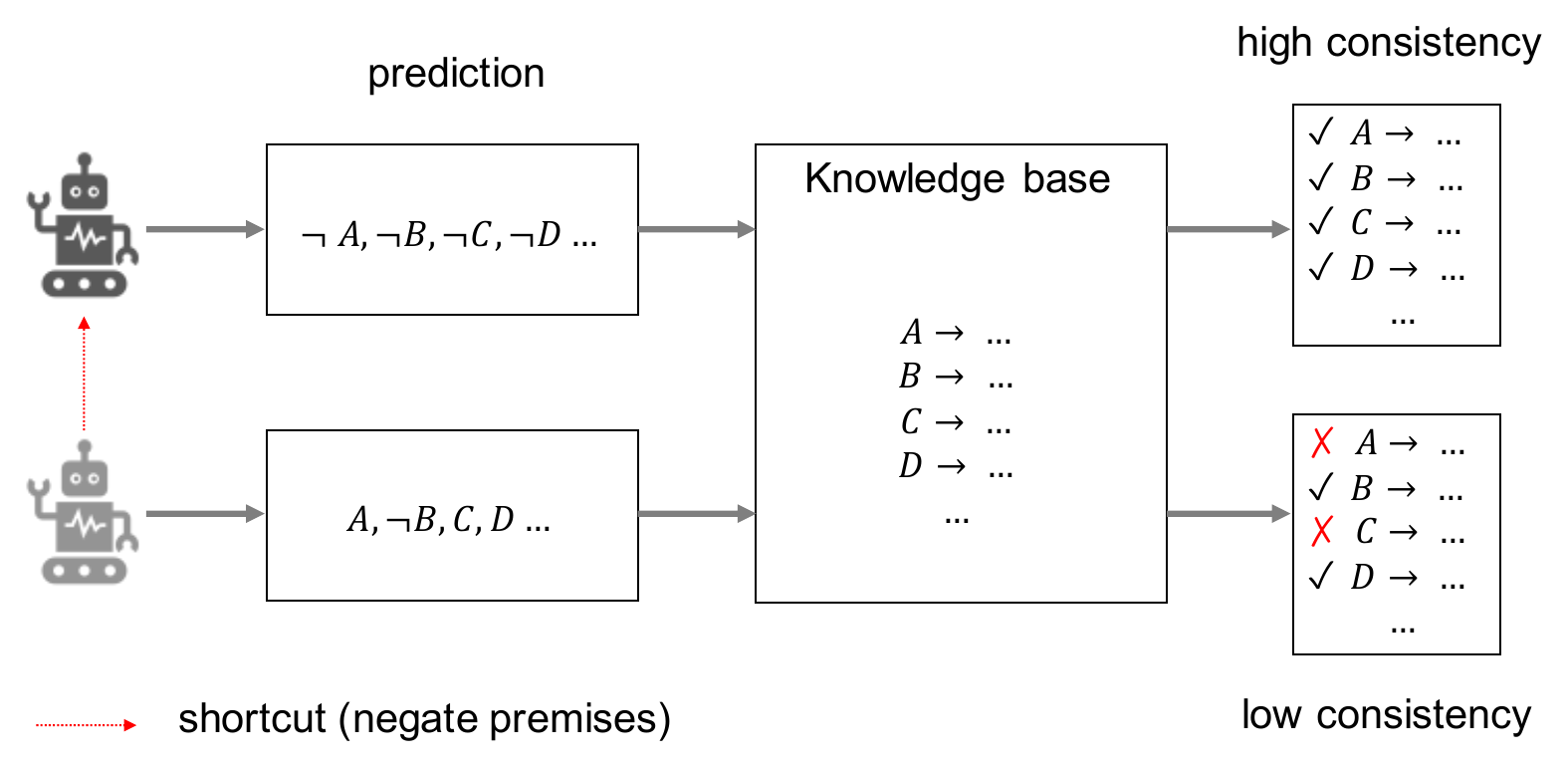}
    \captionsetup{font=small, labelfont=bf}
    \caption{The implication bias is a tendency for NeSy systems to negate the premises of implication rules in order to increase consistency with the knowledge base. }
    \label{fig:ib}
\end{figure}

\bmhead{Organization}
The organization of this paper is as follows. We begin by introducing the background of the field in \cref{sec:preliminary}.
In \cref{sec:ib}, we present a formal definition of \emph{implication bias} and provide a case study to explain this phenomenon.
Our analysis demonstrates that implication bias is prevalent in NeSy logic loss functions derived from fuzzy operators. 
To address this issue, in \cref{sec:rill}, we propose a simple yet effective method, Reduced Implication-bias Logic Loss (RILL). 
Finally, we empirically validate the effectiveness of RILL in \cref{sec:exp} by examining two challenging scenarios: the incomplete knowledge base and insufficient supervised data that we discussed earlier.
This paper's two main contributions are summarized as follows:

\begin{itemize}
    \item We analyze the phenomenon of implication bias caused by differentiable implication operators and identify the loss functions that are susceptible to this bias.
    \item We propose a simple yet effective method called \emph{Reduced Implication-bias Logic Loss (RILL)} to reduce implication bias. Empirical studies show that RILL can achieve significant improvements compared to other forms of logic loss, especially when the knowledge base is incomplete or labeled data is insufficient.
\end{itemize}

\section{Related Work}
Some researchers have attempted to design the structure of neural networks based on logical rules~\citep{towell_KBANN_1994,li_augmenting_2019,MultiplexNet}.
These structures embedded with logical constraints can perform well in specific tasks and satisfy their logic constraints.
However, training such models requires a large amount of data.

Statistical Relational and Neural-Symbolic methods, such as DeepProbLog~\citep{manhaeve_deepproblog_2018}, TensorLog~\citep{cohen2020tensorlog}, and Abductive Learning (ABL)~\citep{zhou2019abductive,dai_abl_2019}, have been proposed to combine neural networks and logical programming.
However, many of them attempt to approximately perform logical reasoning using distributed representations in neural networks, which typically require a tremendous amount of labeled data.
Therefore, for simplicity, we do not consider differentiable operators implemented by parameter models, where a neural network acts as a logical reasoner.
Although the use of neural networks for logical reasoning seems promising, it is beyond the scope of this paper.

A simple way to interact between the logical module and the perception module is to design a logic loss that measures how unsatisfying the model's output is. As pointed out by \citet{DBLP:journals/corr/abs-2108-11451}, there are two perspectives on approximating logical reasoning with differentiable operations. From the perspective of probabilistic logic, \cite{xu_semantic_2018} proposed a semantic loss function to obtain a logic loss.
However, for efficient computation, they need to encode logical rules in Sentential Decision Diagrams (SDD)~\citep{DBLP:conf/ijcai/Darwiche11}, which is an NP-hard problem and expensive for most real-world tasks. 
From the perspective of fuzzy logic, it is natural to use fuzzy operators to design logic losses, as demonstrated by several prior works~\citep{DL2,roychowdhury_regularizing_2021,giannini2019relation,LTN}. These works use fuzzy logic to translate logical rules into loss functions.
Our work is closely related to \cite{van_krieken_survey_2022}, which analyzes different kinds of fuzzy logic operators and discovers a significant imbalance of gradients between premises and consequents in the Reichenbach operator.

\section{Preliminaries}\label{sec:preliminary}
To provide context for our proposed method, we first present some concepts in logic programming.
Next, we introduce the formal definition of continuous-valued logic, which forms the foundation for fuzzy-based logic loss in NeSy systems.
Finally, we explain how logic loss is integrated with task-specific loss in NeSy systems.

\subsection{Closed World Assumption}
The closed-world assumption (CWA)~\citep{Reiter1978} is a presumption that states if something is known to be true, it is considered true. If something is not currently known to be true, it is considered false. This is in contrast to the open-world assumption (OWA)~\citep{OWA}, which holds that lack of knowledge does not imply falsity. The CWA is often used in database contexts, where information not explicitly presented in the database is assumed to be false.
For example:
\begin{quote}
    \centering
    \begin{tabular}{rl}
        Statement:           & Bob and Cam are students. \\
        Question:            & Is Alice a student?       \\
        Closed world answer: & No.                       \\
        Open world answer:   & Unknown.
    \end{tabular}
\end{quote}

\bmhead{Clark's Completion}
Clark's completion aims to address the problem caused by negating premises.
By Clark's Completion, the following knowledge base:
\begin{equation*}
    A\leftarrow A_1, A\leftarrow A_2,\cdots, A\leftarrow A_n,
\end{equation*}
can be rewritten as $A\leftrightarrow (A_1\lor A_2\lor \cdots \lor A_n)$,
thus implementing the closed-world assumption and keeping the soundness of inference~\citep{clark1978negation}.

\subsection{Continuous-valued Logic}
Propositions are denoted as lowercase letters $x,y,z,\dots$, and propositional literals (i.e., $x, \neg x$) stand for $x$ being true or false. A \textit{term} is a variable, a constant, or a function applied to terms. An \textit{atom} (or atomic formula) is either a proposition or a predicate  $p(t_1,\cdots, t_k)$ of arity k where the $t_i$ are terms.
A \textit{formula} is built out of atoms using logical connectives $\neg, \lor, \land, \rightarrow$, respectively. For simplicity, we assume that first-order-logic (FOL) formulas are universally quantified.

In this paper, we focus specifically on the problem caused by implication operators. 
Therefore, we will not discuss other types of operators for the sake of simplicity.

\begin{definition}[Implication Likelihood]
    An implication likelihood, or fuzzy implication, is a function $I(\cdot,\cdot): [0,1]^2\mapsto[0,1]$ that is differentiable and satisfies the following conditions: $I(0,0) = 1$, $I(1,1) = 1$, and $I(1,0) = 0$.
\end{definition}

Implication likelihood is a function that is used to estimate the truth-value of implication rules, which is closely related to fuzzy t-norms~\citep{giannini2019relation}.

\begin{definition}[Logic Likelihood]
    A logic likelihood is a function $s(\cdot,\cdot):\mathbb{F}\times [0,1]^n \mapsto [0,1],$ that estimates the truth value of logic formulas.
\end{definition}

The function $s(\cdot,\cdot)$ takes a logic formula and the corresponding truth-value vector as inputs\footnote{Limited by space, this definition is not complete, the more rigorous definition can be referred to~\cite{klement2013triangular}}. For instance, $s(P\rightarrow Q,(x,y)) = I(x,y)$ evaluates the truth value of the rule $P\rightarrow Q$, where $x$ and $y$ represent the truth values of $P$ and $Q$, respectively.

\begin{definition}[$\delta$-Confidence Monotonic] A logic likelihood is said to be $\delta$-confidence monotonic if its implication likelihood satisfies the following property: for some $y \in [0,\delta)$, $I(x,y)$ is monotonically decreasing with respect to $x$.
\end{definition}

To better understand this definition, let us consider the implication rule $\text{Raven}(x)\rightarrow \text{Black}(x)$. When the truth value of $\text{Black}(x)$ is almost 0 (i.e., $x$ is highly unlikely to be black), the truth value of this rule decreases as the confidence in $\text{Raven}(x)$ increases.

Functions with this property are commonly used in many publications~\citep{LTN,roychowdhury_regularizing_2021,van_krieken_survey_2022}.
For instance, the implication likelihood used in~\cite{LTN,roychowdhury_regularizing_2021} is defined as $I(x,y) = 1-x+x\cdot y$, which is a $1$-confidence monotonic implication likelihood and is widely used in fuzzy-based NeSy systems.
Furthermore, most fuzzy operators that serve as implication likelihoods exhibit this property, which we have shown in the appendix.

\subsection{Logic Loss in NeSy Systems}
\begin{definition}[Logic Loss]
    The logic loss $\ell_{\mathrm{logic}}$ is a function that estimates the inconsistency between a logic formula $r$ and a truth-value vector $\boldsymbol{z}$. It is defined as:
    \begin{equation*}
        \ell_{\mathrm{logic}}(r,\boldsymbol{z}) = g(s(r,\boldsymbol{z})),
    \end{equation*}
\end{definition}
where $g$ is a monotonically decreasing function with $g(1) = 0$.
For example, $g$ can be a negative logarithm function , such as $\ell_{\mathrm{logic}}(r,\boldsymbol{z}) = 1-\log(s(r,\boldsymbol{z})+1)$. In this paper, this function will be chosen as the default.

After we defined the individual logic loss for a single logical formula, the whole logic loss on the knowledge base and dataset (i.e. empirical logical risk) can be expressed as follows,

\begin{align*}
    R_{\mathrm{logic}}(f) = \mathbb{E}_{\mathcal{X}\times\mathrm{KB}}[\ell_{\mathrm{logic}}(r,f(x)],
\end{align*}

In addition to the logic loss, we also need to optimize the loss for a specific task:

\begin{align*}
    R_{\mathrm{task}}(f) = \mathbb{E}_{\mathcal{X}\times\mathcal{Y}}[\ell_{\mathrm{task}}(f(x), y)].
\end{align*}

We can combine these two parts to optimize:

\begin{align*}
    f^* = \arg\max_{f}~[R_{\mathrm{task}}(f) + \lambda\cdot R_{\mathrm{logic}}(f)],
\end{align*}

where $\lambda$ is a hyper-parameter of the training process.

Combining a logic loss with a task-specific loss in machine learning can be mutually beneficial. On the one hand, the task-specific loss can guide the model to achieve good performance and output valid logic primitive facts. On the other hand, the logic loss can help reduce the empirical optimal space and make the model easier to optimize.

However, as we will demonstrate in the next section, applying a logic loss with implication bias can introduce a significant bias into NeSy systems during the training process, which can make the optimization process more challenging.

\section{Implication Bias}\label{sec:ib}
In this section, we will first give a formal definition of implication bias and discuss which types of logic loss are susceptible to this bias.
We will then present a case study to illustrate this bias more clearly.
Finally, we will further discuss implication bias from two perspectives.
\subsection{Definition}
\begin{definition}[Implication Biased]
    A logic loss is said to be implication biased if there exists a small $\delta > 0$ such that its implication likelihood $I(x,y)$ satisfies $\forall x\in(0,\delta), \exists y\in [0,1],$

    \begin{equation*}
        \begin{aligned}
            \frac{\partial I}{\partial x} < 0.
        \end{aligned}
    \end{equation*}

\end{definition}

Using this definition, we can state the following theorem.

\begin{theorem}
    A logic loss $\ell_{\mathrm{logic}}$ that uses a $\delta$-confidence monotonic logic likelihood is implication biased.
\end{theorem}

The proof follows directly from Definition 3. It is worth noting that logic likelihoods with $\delta$-confidence monotonicity are frequently used in many fuzzy-based NeSy systems (as we will demonstrate in the appendix), making this property pervasive in NeSy systems.
\subsection{Case Study}

\begin{figure}[!htb]
    \centering
    \includegraphics[width=.50\linewidth]{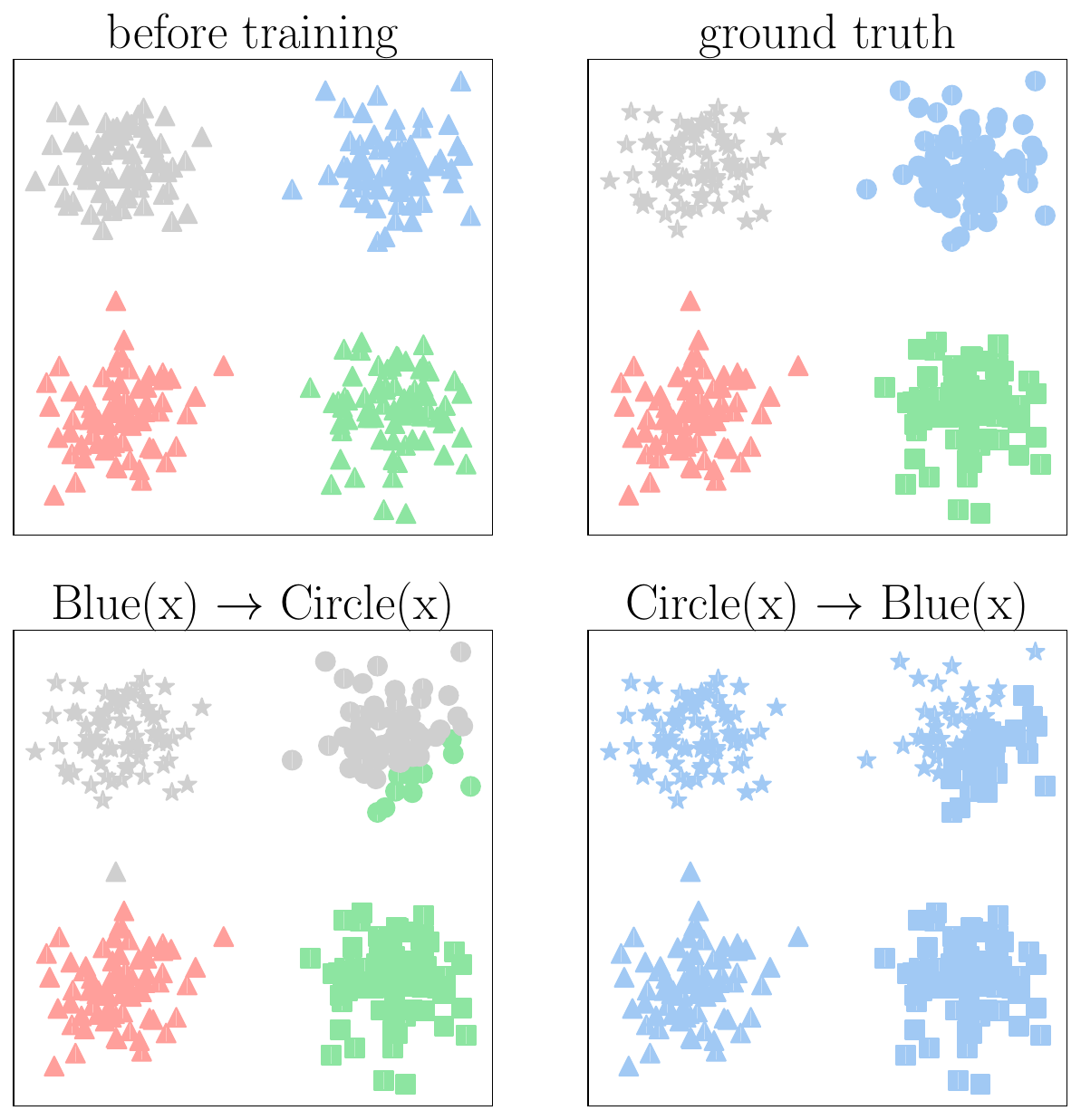}
    \captionsetup{font=small, labelfont=bf}
    \caption{Case study of implication bias. These logic rules are both satisfied their logical constraints.}
    \label{fig:bias_exp}
\end{figure}

This case is designed to illustrate the impact of implication bias during the training process.
The dataset is constructed from four clusters. Each data point in this dataset has two attributes:
\begin{quote}
    \centering
    \begin{tabular}{rll}
        color & $\in$ & $\{\text{Blue, Green, Red, Gray}\}$          \\
        shape & $\in$ & $\{\text{Circle, Square, Triangle, Star}\}.$ \\
    \end{tabular}
\end{quote}
The attributes in our dataset are related to others according to the following logic rules: $\text{Blue}(x)\leftrightarrow \text{Circle}(x),\text{Green}(x)\leftrightarrow \text{Square}(x),\text{Gray}(x)\leftrightarrow \text{Star}(x),\text{Red}(x)\leftrightarrow \text{Triangle}(x)$.

The training set only contains the shape labels. The two different incomplete knowledge bases used in our experiments, shown in \cref{fig:bias_exp}, are ${\text{Blue}(x)\rightarrow \text{Circle}(x)}$ and ${\text{Circle}(x)\rightarrow \text{Blue}(x)}$.

For the task-specific loss, we use CrossEntropy, and for the logic loss, we use the Reichenbach Implication Likelihood, defined as $I(x,y) = 1-x+x\cdot y$. The optimization objective is $R_{\mathrm{task}}(f)+\lambda\cdot R_{\mathrm{logic}}(f)$. The learning model is a one-hidden-layer neural network with two linear classification heads for shape and color.

As shown in the bottom left and right images of \cref{fig:bias_exp}, when the logic loss function is implication biased, the optimized model will avoid predicting any samples as blue, despite the initial model correctly predicting the color labels.
This happens because the majority of the samples in the training set are unrelated to the single implication rule (i.e., they are neither blue nor circular).
These samples will contribute significantly to the gradient towards negating the premise of the implication rule, causing the model to avoid predicting samples as blue.

\begin{figure}[!htb]
    \centering
    \includegraphics*[width=.8\linewidth]{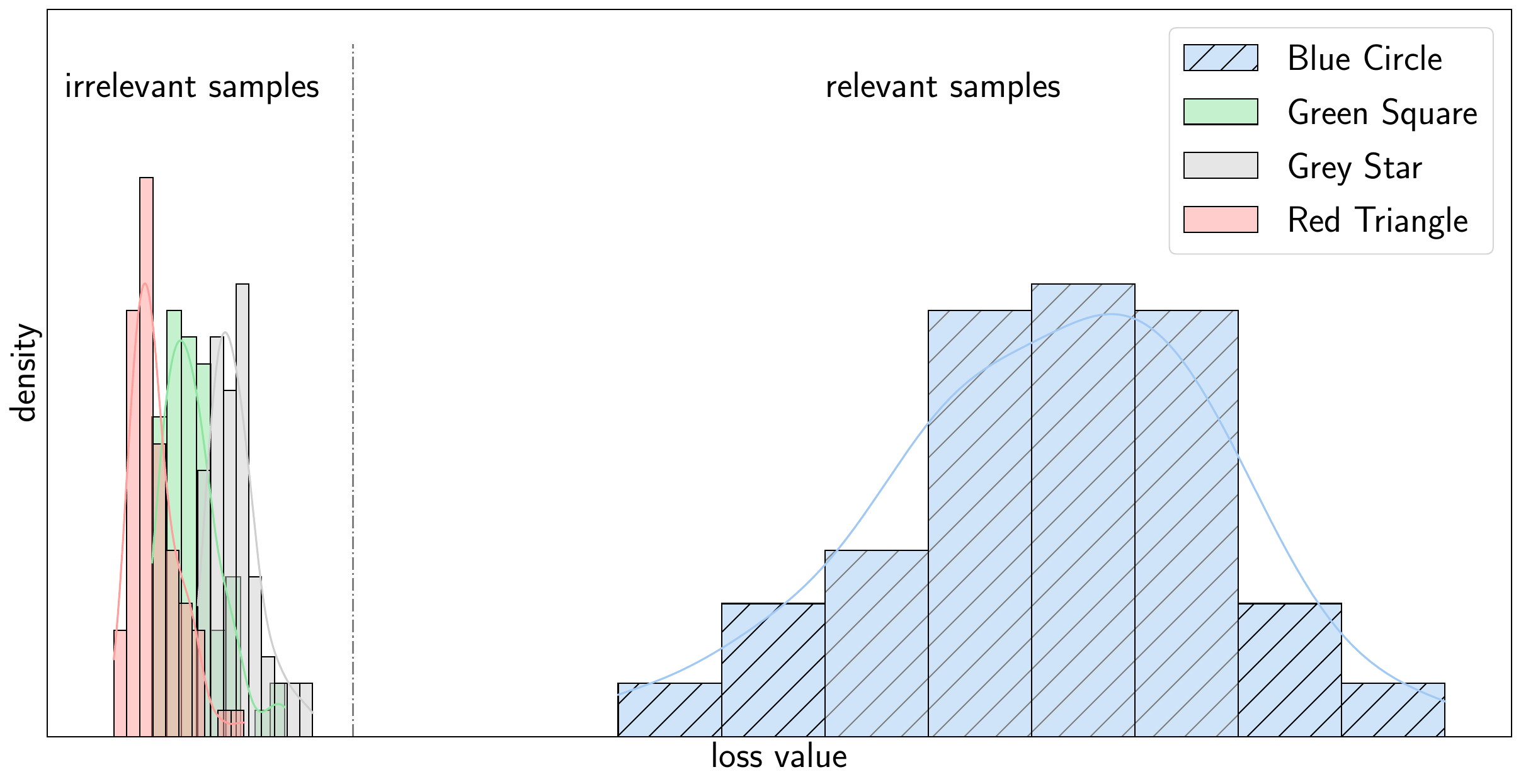}
    \captionsetup{font=small, labelfont=bf}
    \caption{Distribution of different samples given by the loss values, which are computed from rule $\text{Circle}(x)\rightarrow \text{Blue}(x)$.\label{fig:loss_dist}}
\end{figure}

As can be seen in \cref{fig:loss_dist}, there is a significant difference in the distribution of logic loss between relevant and irrelevant samples for the rule $\text{Circle}(x)\rightarrow \text{Blue}(x)$. Even though the loss for irrelevant samples is relatively small, they still contribute a non-negligible amount to the gradient, which can negate the premise of the implication rule.

\bmhead{Remark}
It may be surprising that our model failed to correctly predict the shape labels in the bottom right image of \cref{fig:bias_exp}. This can happen because the logic loss function tends to favor negating the premise (i.e., vacuous truth caused by the implication bias), while the task-specific loss function is designed to favor the data or the goals of the task. As a result, the NeSy system may prioritize satisfying the logic loss function over achieving the goals of the task, leading to sub-optimal performance.

\subsection{Analysis of Implication Bias}
Let us analyze implication bias from two perspectives.
\bmhead{Optimization}
The implication bias in neural networks can hinder the learning process by causing the networks to disregard the premise of implication rules, which are often used as shortcuts in the learning process. This bias may cause the network to perform sub-optimally.
In many cases, a significant portion of the samples in a dataset does not meet the premise of a particular implication rule but rather satisfy the rule through vacuous truth (i.e., by negating the premise)~\citep{DBLP:books/daglib/0076838}. When the logic loss is implication biased, these samples can still significantly contribute to the gradient, exacerbating the tendency to negate the premise of the implication rule and leading to even more biased results.

\bmhead{Logic Programming}
In an open world, negating a model's prediction is undecidable. For instance, consider the statement $\textrm{Raven}(x)\rightarrow \textrm{Black}(x)$. In this case, there are many possible predicates that could be “not raven,” making it difficult to identify an explicit target for improving the model's prediction.

\bmhead{Remark} Implication bias poses a significant challenge for NeSy systems, particularly in the following two common and realistic scenarios:
\begin{itemize}
    \item \emph{Incomplete knowledge base}: In this case, there is missing or uncertain information that is necessary for making accurate inferences. This can make it difficult for the NeSy system to accurately approximate logical reasoning using differentiable operators.
    \item \emph{Insufficient supervised data}: In this case, there is a small amount of labeled data available for training the NeSy system. This can make it difficult for the NeSy system to accurately learn from the data and make accurate predictions.
\end{itemize}

In the above cases, if the NeSy system is using loss functions that are susceptible to implication bias, this can further compound the problem. The bias introduced by these loss functions can cause the model to make inaccurate predictions and degrade its performance.

\section{Reduce Implication Bias Logic Loss~(RILL)}\label{sec:rill}
One possible approach to addressing the implication bias problem is to utilize a likelihood function that does not exhibit this bias.
However, discovering such a function can be challenging and require extensive research.
In this section, we introduce a simpler solution for reducing implication bias.
We will begin by providing an overview of our approach, followed by a detailed explanation of our method and discussions for improved comprehension.

\bmhead{Insight}
\emph{Samples with low confidence in their premise should be assigned less importance:}
\begin{itemize}
    \item Samples with low confidence in their premise tend to have lower loss values, which means that they are less important to the optimization objective compared to other samples. Furthermore, these samples' gradients can introduce a nuisance into the training process, so it is natural to reduce their importance.
    \item Samples with low confidence in their premise can introduce undecidability into the training process. Clark's completion attempts to handle the undecidability of negative predicates by utilizing information explicitly present in the knowledge base. Therefore, it is natural to reduce the importance of information that is not explicitly present.
\end{itemize}

It is worth noting that low-loss value samples consist of two types. 
The first type includes samples that satisfy a specific rule and do not require backward tuning of the model's parameters, leading to a low loss value. 
The second type comprises samples with low confidence in their premise, which we should assign less importance to.
Therefore, we can consider low-loss value samples as irrelevant samples regarding a specific rule. Based on this, assigning importance to the samples can be straightforward, as we can assign different levels of importance based solely on their loss value.

Consider all samples' individual logic loss on a formula $r$ and rank them by their loss value~(from low to high):

\begin{equation*}
    \underbrace{\ell_1^r, \ell_2^r, \cdots, \ell_t^r}_{\text{weak samples}}, \cdots, \ell_N^r.
\end{equation*}

The ``weak samples'' referred to in the equation are those with low loss values, such as irrelevant samples depicted in \cref{fig:loss_dist}. 
These weak samples are generally considered less reliable than the other samples when applying the given rule. 
To distinguish between weak and non-weak samples, we can use a threshold $\epsilon$ to divide them.

Without losing generality, we can rewrite the risk of logic in the following form:

\begin{equation*}
    R_{\textrm{logic}} = \sum_{r\in\textrm{KB}} \mathcal{A}(\left\{ \ell_i^r | i \in [N] \right\}),
\end{equation*}

where the aggregator $\mathcal{A}$ is typically used to average the losses. One way to reduce bias in the results is to redefine the aggregator to give less weight to weak samples.
Let $\mathcal{L}^r$ be the set of losses $\left\{ \ell_i^r | i \in [N] \right\}$. We propose three possible aggregators as follows.

\bmhead{Hinge}
One option for reducing bias in the results is to use a hinge-like aggregator, which ignores the losses of weak samples.

\begin{equation*}
    \mathcal{A}_{\textit{hinge}}(\mathcal{L}^r) = \textrm{Average}(\left\{ \ell_i^r | \ell_i^r \geq \epsilon, i \in [N] \right\})
\end{equation*}

This aggregator can be viewed as a form of sample selection, as it only includes the loss values of non-weak samples in the computation. It is similar to the RAMP loss function~\citep{phoungphol2012robust}, which also only considers a subset of samples for robust optimization.

\bmhead{L2 Smooth}
Since we want to reduce the bias caused by weak samples' gradients of biased logic loss, it is natural to smooth by multiplying itself.

\begin{equation*}
    \mathcal{A}_{\textit{l2}}(\mathcal{L}^r) = \textrm{Average}(\left\{ (\ell_i^r)^2 | i \in [N] \right\})
\end{equation*}

The aggregator in this method smooths the biased logic loss by multiplying it with itself, giving lower-loss samples a smoother gradient and reducing the bias of weak samples. This can be seen in the equation $\nabla \ell^2 = 2\ell \nabla \ell$.

\bmhead{L2+Hinge} Kind of mixing the above aggregators:

\begin{equation*}
    \begin{aligned}
        \mathcal{A}_{\textit{l2+hinge}}(\mathcal{L}^r) =
        \textrm{Average}(\left\{ \mathcal{T}(\ell_i^r) | i \in [N] \right\}),
    \end{aligned}
\end{equation*}where
\begin{equation*}
    \mathcal{T}(\ell) = \ell^2\cdot\mathbb{I}[\ell \leq \epsilon] + \ell\cdot \mathbb{I}[\ell >\epsilon].
\end{equation*}

\bmhead{Remark}
The solutions described above aim to reweight the samples based on their loss value, which reflects their relevance to an implication rule. While we have presented three specific aggregators for achieving this, it is worth noting that other approaches that similarly aim to reduce the influence of weak samples could also be effective in addressing this issue.

\bmhead{Discussion}
RILL reduces the uncertainty associated with negative information, similar to Clark's completion. While RILL does not alter the information in the knowledge base but instead reduces the importance of weak samples. This reduction causes the model to pay more attention to samples that are more relevant to the given rule.

On the other hand, explicitly applying Clark's completion in a NeSy system may not be helpful. In this setting, replacing $\rightarrow$ with $\leftrightarrow$ would change the information in the knowledge base, potentially leading to incorrect information. While Clark's completion does not affect the soundness of the system in logic programming, it may not be the case in a data-driven approach such as NeSy.
More discussion on the relationship between RILL and Clark's completion can be found in the appendix.

\section{Empirical Study}\label{sec:exp}
In this section, we will discuss two common challenges that occur in real-world scenarios: incomplete knowledge bases and insufficient supervised data. 
We will begin by briefly introducing the compared methods. 
Next, we will set up the experiments.
Finally, we empirically validate the performance of RILL and other compared methods under these settings.

\subsection{Setting Up}
%这里说明对比方法, 为什么选它们
The compared methods come from two mainstream perspectives of NeSy systems. From the fuzzy logic perspective\footnote{A detailed empirical study on other kinds of fuzzy operators can be seen in the appendix.}, the logic loss derived from Reichenbach operators outperforms other methods, as reported in~\cite{van_krieken_survey_2022}. From the probabilistic logic perspective, we chose the semantic loss~(SL for short)~\citep{xu_semantic_2018} due to its well-defined nature and efficiency. The following definition is taken from~\cite{xu_semantic_2018}.

\begin{definition}[Semantic Loss]
    Let $\mathrm{p}$ be a vector of probabilities, one for each variable in $\mathbf{X}$, and let $\alpha$ be a sentence over $\mathbf{X}$. The semantic loss between $\alpha$ and $\mathrm{p}$ is defined as:

    \begin{equation*}
        \mathrm{L}^{\mathrm{s}}(\alpha, \mathrm{p}) \propto-\log \sum_{\mathbf{x} \models \alpha} \prod_{i: \mathbf{x} \models X_{i}} \mathrm{p}{i} \prod{i: \mathbf{x} \models \neg X_{i}}\left(1-\mathbf{p}_{i}\right).
    \end{equation*}
\end{definition}

In the optimization process, the logic loss is calculated based on rules in the knowledge base and the training data, which includes both labeled and unlabeled data. The task-specific loss, on the other hand, is only calculated using the labeled data in the training set.

Since this paper focuses on NeSy systems based on loss functions, we do not consider other types of NeSy systems. Additionally, we present a method that does not rely on any information from the knowledge base, which we refer to as ``Vanilla'' in the following table. More details about the experiments and their configurations can be found in the appendix.

\subsubsection{Task 1 Addition Equations}\label{sec:exp:add}

Inspired by~\citet{MultiplexNet}, we constructed a semi-supervised-like task that uses a small number of labeled data and a large amount of well-structured unlabeled data to train a machine learning model.
Specifically, the data is structured by additive equations, as shown in \cref{fig:add_exp}.
This structure can be applied to datasets with ten classes, not just digit datasets.
In our experiments, we use MNIST~\citep{deng2012mnist}, FashionMNIST~\citep{xiao2017fmnist}, and CIFAR-10~\citep{krizhevsky2009cifar} as basic datasets, which we organize as illustrated in \cref{fig:add_exp}.
We use three-layer Multi-layer Perceptron (MLP) models for MNIST/FashionMNIST and ResNet9~\citep{he2016deep} for CIFAR-10. The task-specific loss for this task is CrossEntropy.
\begin{figure}[!htb]
    \centering
    \includegraphics*[width=.5\linewidth]{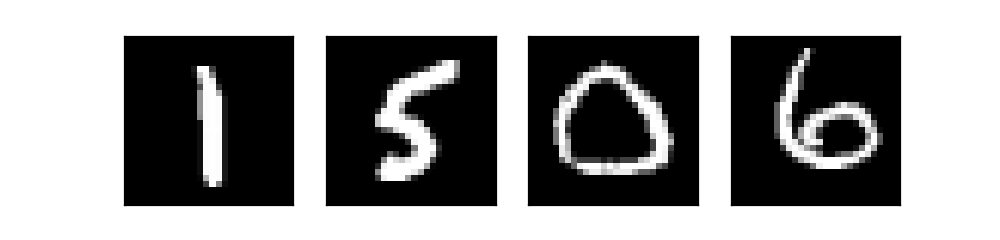}
    \captionsetup{font=small, labelfont=bf}
    \caption{Addition Equation Experiment, data are structured with equation constraints.
    All adopted equations in this experiment consist of four digits.
    For instance, in this picture, $1+5=06$.\label{fig:add_exp}}
\end{figure}

\bmhead{Knowledge Base} Addition equations with four digits have $10\times 10 = 100$  rules.
For example, rules $5+2=07,5+3=08$ can be represented as:
\begin{quote}
    \centering
    \begin{tabular}{cc}
        $\text{1stD5}(x_1) \land \text{2ndD2}(x_2) \rightarrow$ & $\text{3thD0}(x_3) \land \text{4thD7}(x_4)$   \\
        $\text{1stD5}(x_1) \land \text{2ndD3}(x_2) \rightarrow$ & $\text{3thD0}(x_3) \land \text{4thD8}(x_4). $ \\
    \end{tabular}
\end{quote}

Task 1 is well suited to validate the performance of NeSy systems in the scenario of incomplete knowledge bases since its structure allows for strong constraints to be imposed during learning.
Specifically, logic loss of this task can correct misclassified labels when most digits are recognized correctly.
This adaptability to varying amounts of labeled data also makes it suitable for scenarios with limited supervised data.

\subsubsection{Task 2 Hierarchical Classification}

CIFAR-100~\citep{krizhevsky2009cifar} consists of 100 classes and 20 super-classes (SC). For example, an image can be classified with a class label of \textit{maple} and a super-class label of \textit{trees}. In this experiment, we utilize the relationships between classes and super-classes. We adopt the WideResNet28-8~\citep{zagoruyko2016wide} model as the backbone, with two linear classification heads for the class label and super-class label. The task-specific loss for this experiment is CrossEntropy.

\bmhead{Knowledge Base} Knowledge rules in this task are relationships between classes and super-classes, such as:
\begin{quote}
    \centering
    \begin{tabular}{rl}
        $\text{Aquarium fish}(x)\rightarrow$ & $\text{Fishes}(x)$   \\
        $\text{Shark}(x)\rightarrow$         & $\text{Fishes}(x). $ \\
    \end{tabular}
\end{quote}

Task 2 has a weaker knowledge base than Task 1. In this task, even if the model correctly predicts the super-class, the knowledge base does not provide precise information on how to correct wrongly predicted sub-classes. Therefore, this task mainly focuses on scenarios with insufficient supervised data.

\subsection{Incomplete Knowledge Base}
When the knowledge base is incomplete, implication bias can be a significant issue for NeSy systems. This is because the model may not have access to all the relevant knowledge for making accurate reasoning and may instead rely more on vacuous truth. As a result, the model may perform poorly.

\begin{figure*}[!htb]
    \centering
    \captionsetup{font=small, labelfont=bf}
    \begin{subfigure}[t]{0.32\textwidth}
        \centering
        \includegraphics*[width=\textwidth]{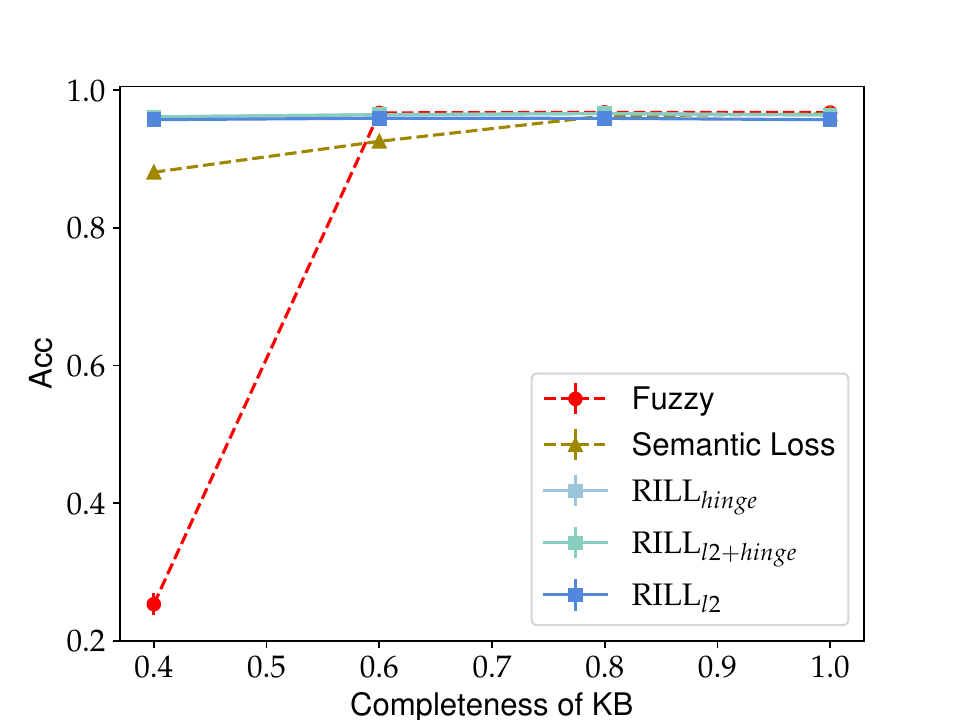}
        \caption{\small Add-MNIST\label{fig:add_mnist_kbs}}
    \end{subfigure}
    \begin{subfigure}[t]{0.32\textwidth}
        \centering
        \includegraphics*[width=\textwidth]{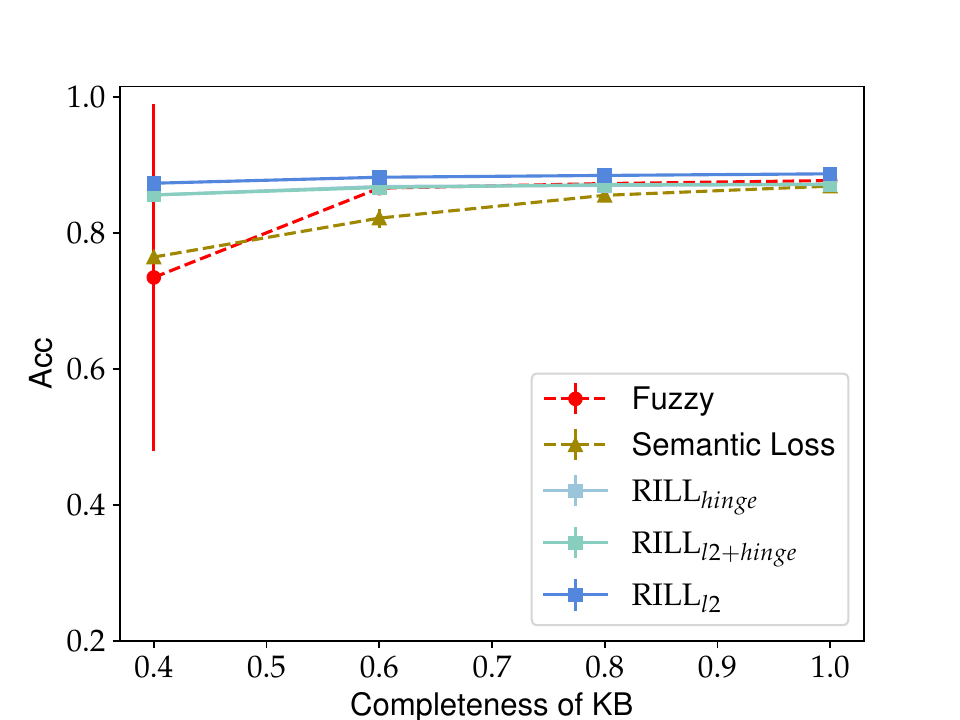}
        \caption{\small Add-FashionMNIST\label{fig:add_fmnist_kbs}}
    \end{subfigure}
    \begin{subfigure}[t]{0.32\textwidth}
        \centering
        \includegraphics*[width=\textwidth]{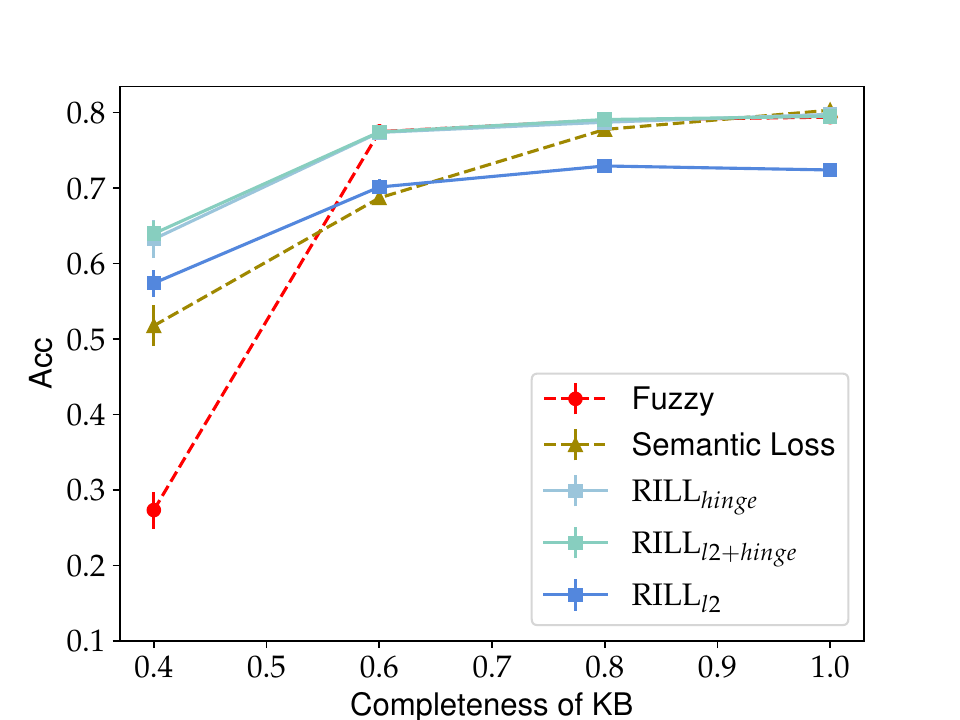}
        \caption{\small Add-CIFAR10\label{fig:add_CIFAR10_kbs}}
    \end{subfigure}
    \caption{\label{fig:add_kbs}\textbf{Effect of incomplete knowledge base}: {\small In this case, models may not have access to all the relevant knowledge for making accurate reasoning.}}
\end{figure*}
The completeness of knowledge bases ranges from 100\% to 40\%, i.e., the number of rules in knowledge bases varies from 100 to 40.
The size of the labeled dataset used in this task is 100 for MNIST and FashionMNIST and 2000 for CIFAR-10.
In Figure \ref{fig:add_kbs}, it is illustrated that as the degree of completeness decreases, the model's performance tends to decline as well.

\bmhead{Analysis} 
The RILL approach is more stable and performs better than other methods, particularly when the knowledge base is incomplete. This is because RILL assigns less importance to weak samples, which helps reduce the impact of implication bias. Thus, RILL is a useful approach for dealing with incomplete knowledge bases and mitigating the effects of implication bias.
\subsection{Insufficient Supervised Data}

Insufficient supervised data can hinder the NeSy system's ability to learn and make accurate predictions. This is because there may not be enough examples to learn from and generalize to new cases. Using loss functions that are prone to implication bias can exacerbate this problem. The experimental results for this scenario are shown in \cref{fig:add_lds}, \cref{tab:cifar10_size}, and \cref{tab:cifar100_size_acc}.
\begin{figure*}[!htb]
    \centering
    \captionsetup{font=small, labelfont=bf}
    \begin{subfigure}[t]{0.32\textwidth}
        \centering
        \includegraphics*[width=\textwidth]{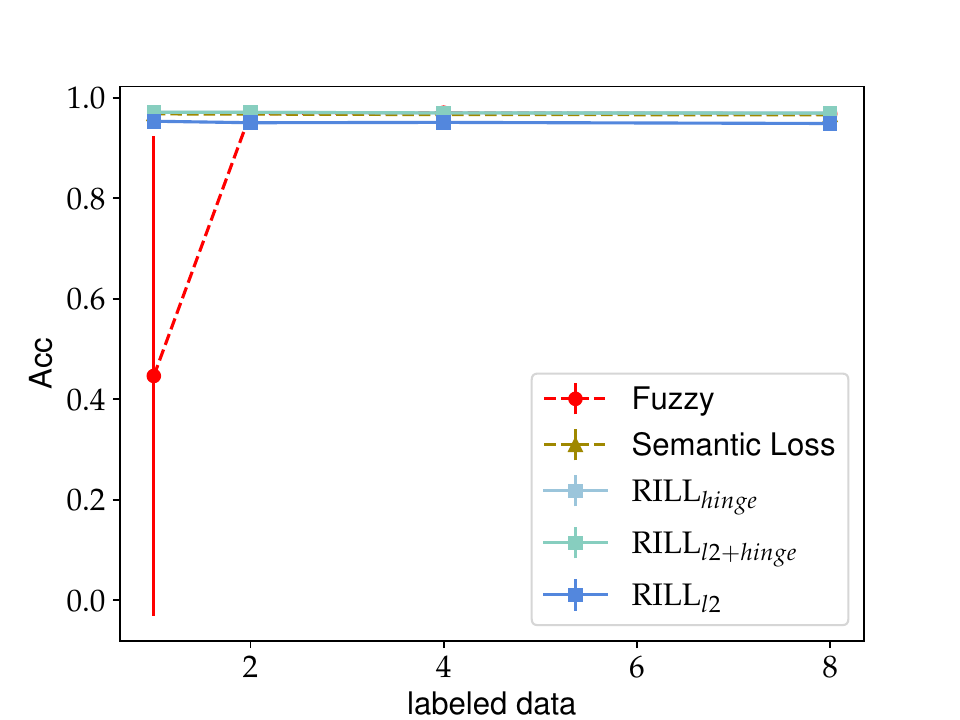}
        \caption{Add-MNIST\label{fig:add_mnist_ld}}
    \end{subfigure}
    \begin{subfigure}[t]{0.32\textwidth}
        \centering
        \includegraphics*[width=\textwidth]{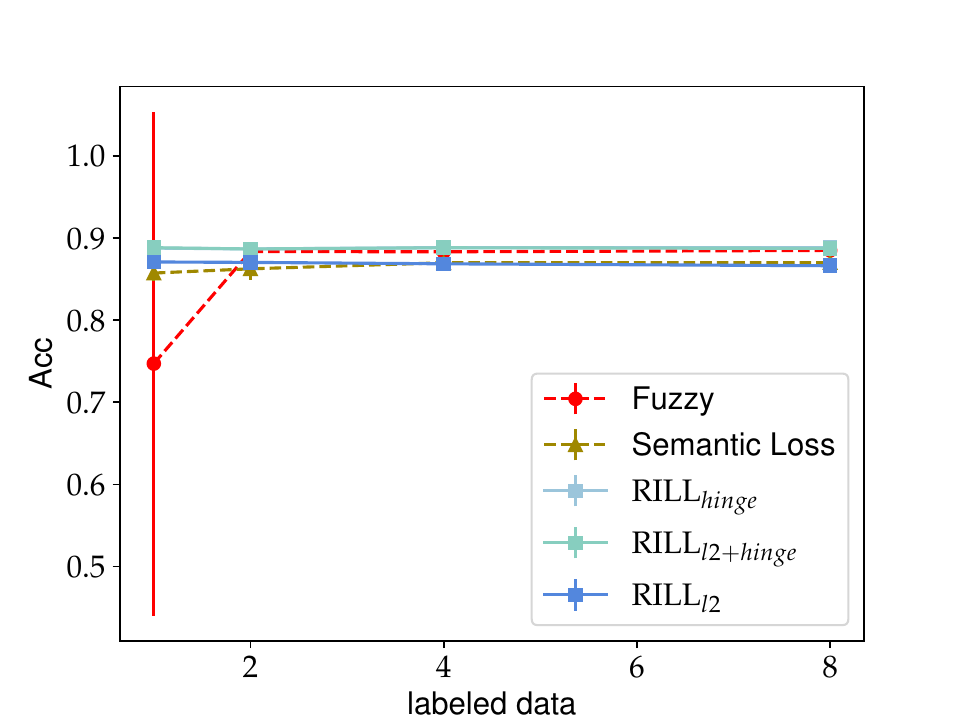}
        \caption{Add-FashionMNIST\label{fig:add_fmnist_ld}}
    \end{subfigure}
    \begin{subfigure}[t]{0.32\textwidth}
        \centering
        \includegraphics*[width=\textwidth]{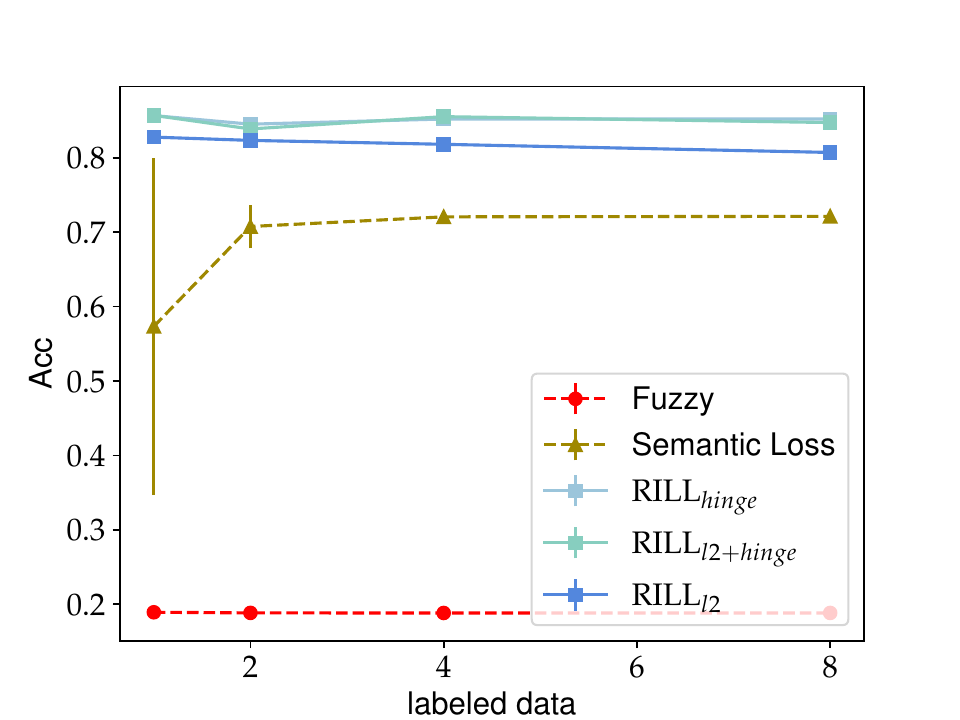}
        \caption{Add-CIFAR10\label{fig:add_CIFAR10_ld}}
    \end{subfigure}
    \caption{\label{fig:add_lds}\textbf{Effect of insufficient supervised data}: In this case, the system may not have enough examples to learn from and generalize to new cases}
\end{figure*}

\begin{table*}[!htb]
    \captionsetup{font=small, labelfont=bf}
    \centering
    \scalebox{1}{
        \begin{tabular}{lcccc}
            \toprule
            supervised data           & 8                         & 4                         & 2                         & 1                         \\
            \midrule
            Fuzzy                     & 0.1891$\pm$0.001          & 0.1906$\pm$0.004          & 0.1902$\pm$0.002          & 0.1904$\pm$0.001          \\
            SL                        & 0.7266$\pm$0.009          & 0.7200$\pm$0.003          & 0.7316$\pm$0.049          & 0.6785$\pm$0.024          \\
            $\textrm{RILL}{l2}$       & 0.8502$\pm$0.005          & \textbf{0.8560$\pm$0.003} & 0.8473$\pm$0.015          & \textbf{0.8582$\pm$0.004} \\
            $\textrm{RILL}{l2+hinge}$ & 0.8139$\pm$0.012          & 0.8202$\pm$0.004          & 0.8265$\pm$0.005          & 0.8316$\pm$0.006          \\
            $\textrm{RILL}_{hinge}$   & \textbf{0.8548$\pm$0.005} & \textbf{0.8573$\pm$0.011} & \textbf{0.8481$\pm$0.007} & \textbf{0.8581$\pm$0.003} \\
            \bottomrule
        \end{tabular}
    }
    \caption{\textbf{Effect of insufficient supervised data}: Task1, Add-CIFAR-10\label{tab:cifar10_size}, showing accuracies for different amounts of supervised data. Experiments were repeated five times.
    The boldface entries denote the best results, according to the Wilcoxon signed-rank test at the 5\% significance level.}
\end{table*}

\bmhead{Task 1} The size of the labeled dataset decreases from eight to one sample per class. As shown in \cref{fig:add_lds}, the fuzzy logic loss struggles to improve the model's performance when the dataset is small, due to the difficulty of making accurate predictions and the effects of implication bias, particularly in the more challenging Add-CIFAR-10 task~(c.f., \cref{tab:cifar10_size}). However, RILL performs much better because it assigns less weight to weak samples, which helps mitigate the impact of implication bias. It's worth noting that RILL was able to achieve an 86\% accuracy on CIFAR-10 using only one labeled sample per class.

\begin{table*}[!htb]
    \centering
    \captionsetup{font=small, labelfont=bf}
    \scalebox{1}{
        \begin{tabular}{lcc|cc}
            \toprule
            \multirow{2}{*}{supervised data} & \multicolumn{2}{c}{10}    & \multicolumn{2}{c}{1}                                                             \\
                                             & Acc                       & SC-Acc                    & Acc                       & SC-Acc                    \\
            \midrule
            Fuzzy                            & 0.4449$\pm$0.010          & 0.7585$\pm$0.014          & 0.2291$\pm$0.004          & 0.7484$\pm$0.014          \\
            SL                               & 0.4310$\pm$0.034          & 0.7462$\pm$0.037          & 0.2281$\pm$0.007          & 0.7384$\pm$0.021          \\
            Vanilla                          & 0.4520$\pm$0.004          & 0.7670$\pm$0.010          & 0.2297$\pm$0.007          & 0.7415$\pm$0.014          \\
            $\text{RILL}_{l2+hinge}$         & 0.4509$\pm$0.002          & 0.7649$\pm$0.008          & \textbf{0.2363}$\pm$0.009 & \textbf{0.7575}$\pm$0.004 \\
            $\text{RILL}_{l2}$               & \textbf{0.4538}$\pm$0.003 & 0.7699$\pm$0.005          & \textbf{0.2357}$\pm$0.010 & 0.7449$\pm$0.014          \\
            $\text{RILL}_{hinge}$            & \textbf{0.4562}$\pm$0.004 & \textbf{0.7750}$\pm$0.006 & 0.2334$\pm$0.006          & 0.7463$\pm$0.014          \\
            \bottomrule
        \end{tabular}
    }
    \caption{\textbf{Effect of insufficient supervised data}: Task2, Hierarchical-Classification \label{tab:cifar100_size_acc} showing accuracies for different amounts of supervised data. Experiments were repeated five times. The boldface entries denote the best results  according to the Wilcoxon signed-rank test at the 5\% significance level.}
\end{table*}
\bmhead{Task 2} The size of the class-labeled dataset decreases from ten to one sample per class. As shown in \cref{tab:cifar100_size_acc}, RILL maintains relatively high performance (both acc and sc-acc) and stability (indicated by std) compared to other methods, particularly when the labeled dataset is small. This highlights the importance of reducing implication bias.

\bmhead{Remark}
One may be surprised to see that the performance of our model slightly improved as the size of the labeled data decreased (see \cref{fig:add_lds}). This may be similar to the smooth label effect discussed in \cite{muller2019does}. However, this phenomenon only appears in Task 1 and not in Task 2 because the knowledge base used in Task 2 is weaker and does not provide precise instructions for correcting misclassified sub-classes. This is likely why RILL is less effective in Task 2.

\section{Discussion and Limitation}
\bmhead{Discussion}
The results of the experiments show that $\textrm{RILL}_{l2}$ performs worse than the other types of $\textrm{RILL}$. This is because the $l2$ aggregator gives equal weight to all samples by multiplying their loss values, while the hinge and l2+hinge aggregators use a hard threshold to reduce the influence of weak samples. Therefore, the latter two aggregators can reduce bias more effectively if an appropriate threshold is chosen.

Implication bias may not be a problem in NeSy systems if the following conditions are met: \emph{complete knowledge base}, \emph{sufficient supervised data}, and \emph{robust training methods}. The extent to which implication bias affects a NeSy system will depend on the task and training methods. In general, it is important to consider the potential impact of implication bias and take steps to mitigate it if necessary.

\bmhead{Limitation}
This work has two limitations. First, it only computes the logic loss from each individual logic rule in the knowledge base and does not consider the complex interactions between different rules. Future work could involve measuring complex reasoning processes using loss functions. Second, the discussion of weak samples requires further investigation. This topic is related to learning with noisy labels~\citep{natarajan2013learning}, and identifying these samples remains a key problem in this field.

\section{Conclusion}
This paper analyzes implication bias, which is a tendency for NeSy systems to shortcut the logic of implication rules by negating the premise. The paper discusses the cause and negative effects of implication bias and confirms its existence through experiments. We propose Reduced Implication-bias Logic Loss (RILL) as a solution. RILL reduces the uncertainty of negative information by lowering the importance of weak samples and causing the model to pay more attention to relevant samples. Empirical studies show that RILL can improve performance and increase robustness compared to other forms of logic loss, especially in cases of incomplete knowledge bases or insufficient supervised data.

\begin{appendices}
    \section{Discussion with other kinds of fuzzy operators}
    In this section, we investigate different types of fuzzy operators and experimentally validate their effectiveness while maintaining the implication biased. When selecting fuzzy operators for NeSy, it is crucial to consider their smoothness and ease of optimization. Therefore, we concentrate on commonly used fuzzy operators that possess these desirable properties. Fuzzy operators that do not have a smooth gradient will be disregarded in our analysis.
    \subsection{Analysis}
    \bmhead{Sigmoidal}
    \citet{van_krieken_survey_2022} proposed an operator which is smoothed Reichenbach operator by a sigmoid function.
    This implication likelihood function is defined as follows:

    \begin{equation*}
        \sigma_{I}(p,q)=  \frac{1+e^{-s\left(1+b_{0}\right)}}{e^{-b_{0} s}-e^{-s\left(1+b_{0}\right)}} \cdot
        \left(\left(1+e^{-b_{0} s}\right) \cdot \sigma\left(s \cdot\left(I(p,q)+b_{0}\right)\right)-1\right),
    \end{equation*}

    where $s>0,b_0\in\mathbb{R},\sigma(\cdot) = \frac{1}{1+e^x}$ denotes the sigmoid function.
    Substituting $d=\frac{1+e^{-s\left(1+b_{0}\right)}}{e^{-b_{0} s}-e^{-s\left(1+b_{0}\right)}}, h =(1+e^{-s\cdot b_0}),f=\sigma \left(s \cdot\left(I(p,q)+b_{0}\right)\right)$, we find:

    \begin{equation*}
        \frac{\partial \sigma_{I}(p,q)}{\partial I(p,q)}=d \cdot h \cdot s \cdot f\cdot(1-f)>0.
    \end{equation*}

    That is to say when $I(p,q)$ is $\delta$-Confidence Monotonic, $\sigma_I(p,q)$ will be $\delta$-Confidence Monotonic too.
    This means logic loss derived from $\sigma_I$ will become implication biased.

    \bmhead{Łukasiewicz}
    Łukasiewicz implication likelihood~\citep{Cignoli2007} was defined as follows:

    \begin{equation*}
        I_{LK}(p,q) = \min(1-p+q,1).
    \end{equation*}

    It is easy to calculate the gradient of this likelihood,
    $\frac{\partial I_{LK}}{\partial p} = -1 \cdot \mathbb{I}[p>q]$.
    It turns out this is also implication biased.

    \bmhead{G$\ddot{\text{o}}$del} G$\ddot{\text{o}}$del implication likelihood~\citep{Paad2016RelationB} was defined as follows:

    \begin{equation*}
        I_{G}(p,q) = \max(1 - p, q).
    \end{equation*}

    Also, the gradient of this likelihood $\frac{\partial I_{G}}{\partial p} = -1 \cdot \mathbb{I}[p+q<1]$ indicates this operator is implication biased.

    \bmhead{Nilpotent} Nilpotent implication likelihood~\citep{Gerla2011} was defined as follows:

    \begin{equation*}
        I_{N}(p,q) = \left\{
        \begin{aligned}
             & 1,           & \mathrm{if}\quad p\leq q. \\
             & \max(1-p,q), & \mathrm{otherwise}.
        \end{aligned}
        \right.
    \end{equation*}

    Also, the gradient of this likelihood $\frac{\partial I_{N}}{\partial p} = -1 \cdot \mathbb{I}[(p+q<1)\&(p > q)]$ indicates this operator is implication biased.

    \subsection{Empirical Study}
    \begin{figure}[!htb]
        \centering
        \captionsetup{font=small, labelfont=bf}
        \begin{subfigure}[t]{0.43\textwidth}
            \centering
            \includegraphics*[width=\linewidth]{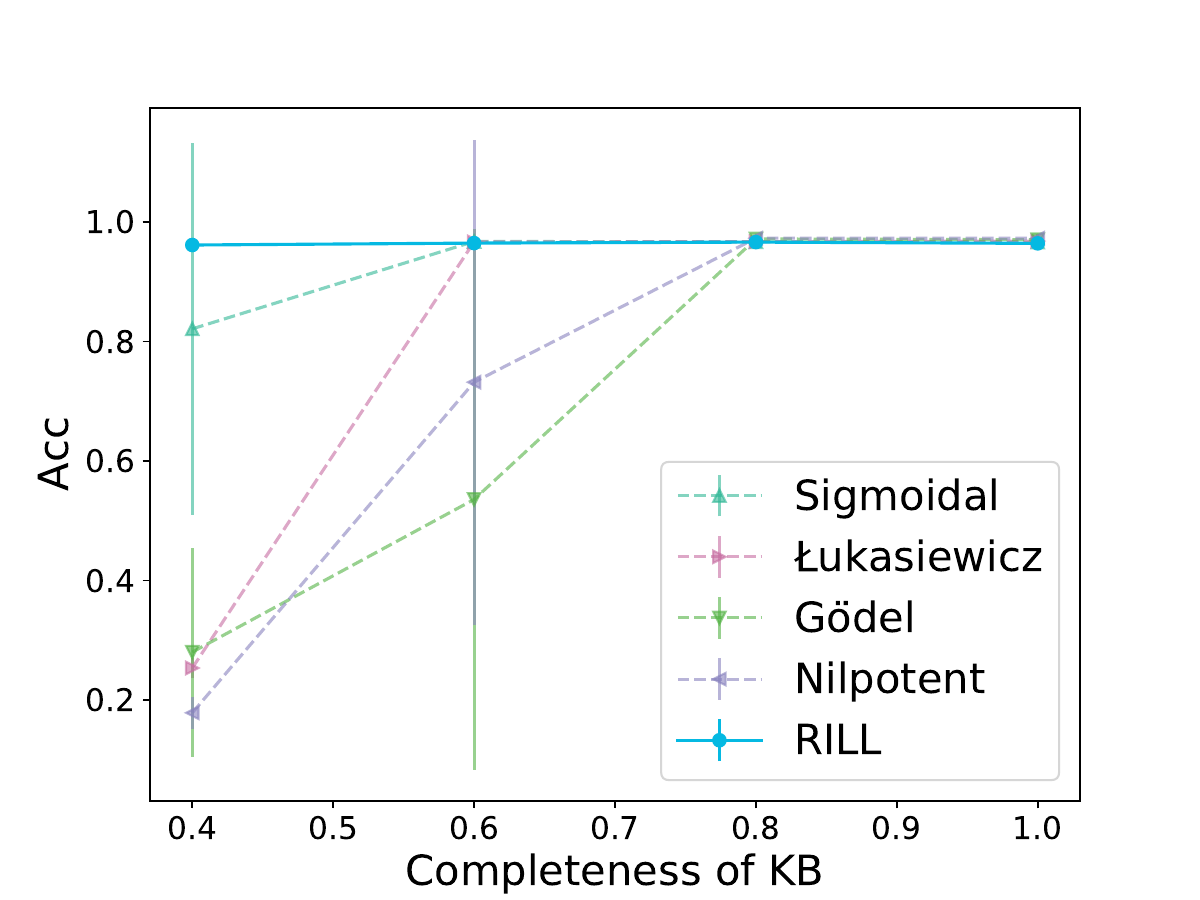}
            \caption{ Incomplete knowledge base\label{app:fig:mnist_fuzzy}}
        \end{subfigure}
        \begin{subfigure}[t]{0.43\textwidth}
            \centering
            \includegraphics*[width=\linewidth]{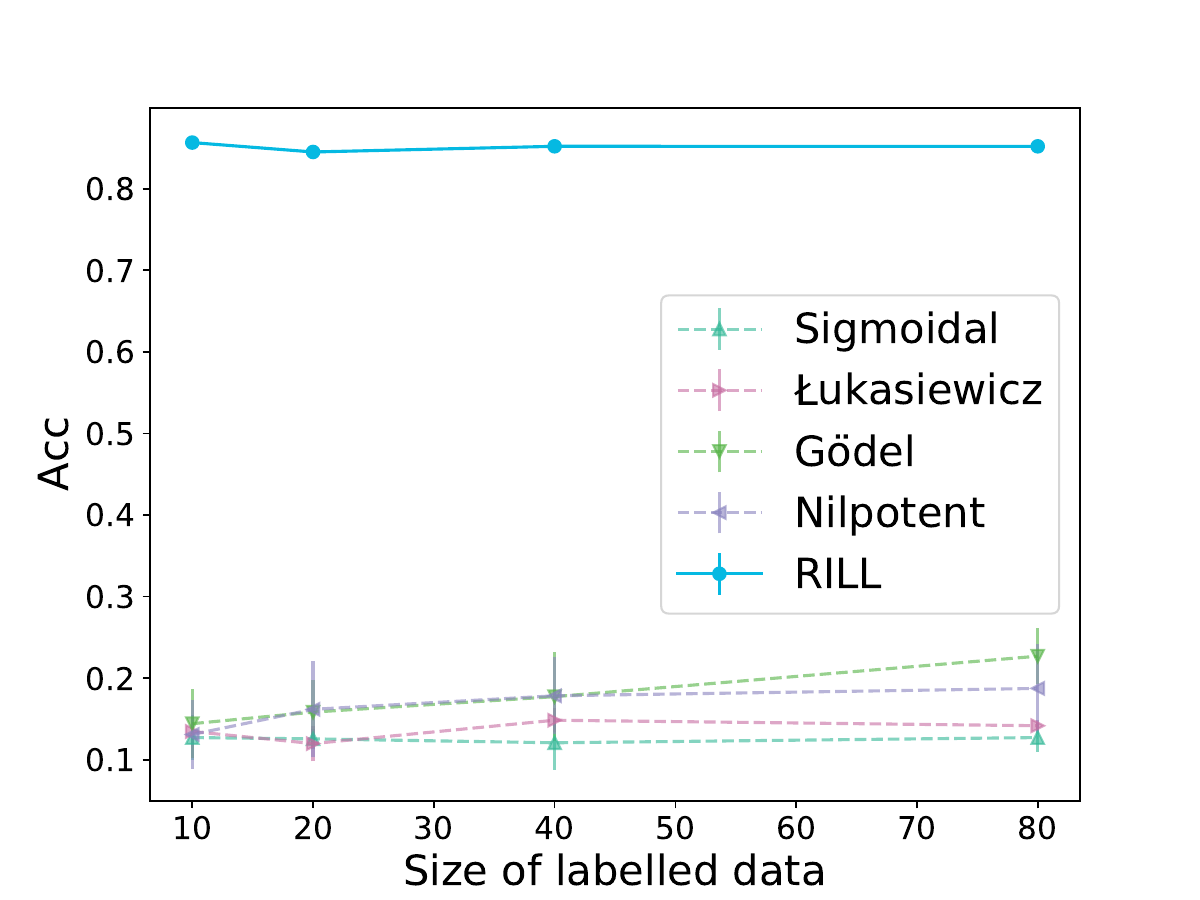}
            \caption{ Insufficient supervised data\label{app:fig:cifar_fuzzy}}
        \end{subfigure}
        \caption{Verification of implication bias across various fuzzy operators is crucial. The empirical evidence illustrated in the preceding images substantiates our assertion that this phenomenon is pervasive.\label{app:fig:fuzzy_op}}
    \end{figure}
    
In this section, we present the empirical results of the above-analyzed operators for validation, with a particular emphasis on incomplete knowledge bases~(especially Add-MNIST) and insufficient labeled data~(especially Add-CIFAR10) scenarios. All experiments follow the same settings used in \cref{sec:exp:add}.

As shown in \cref{app:fig:fuzzy_op}, implication bias significantly harms the performance of the model, especially when the knowledge base is incomplete or the amount of supervised information is not sufficient, which further supports our above analysis.

    \section{Discussion about Clark's Completion}
    RILL and Clark's completion both aim to reduce the uncertainty of negative information. While RILL does not alter the information in the knowledge base, but instead reduces the importance of weak samples, causing the model to pay more attention to samples that are more relevant to the given rule.
    \begin{figure}[!hbt]
        \centering
        \includegraphics*[width=.5\linewidth]{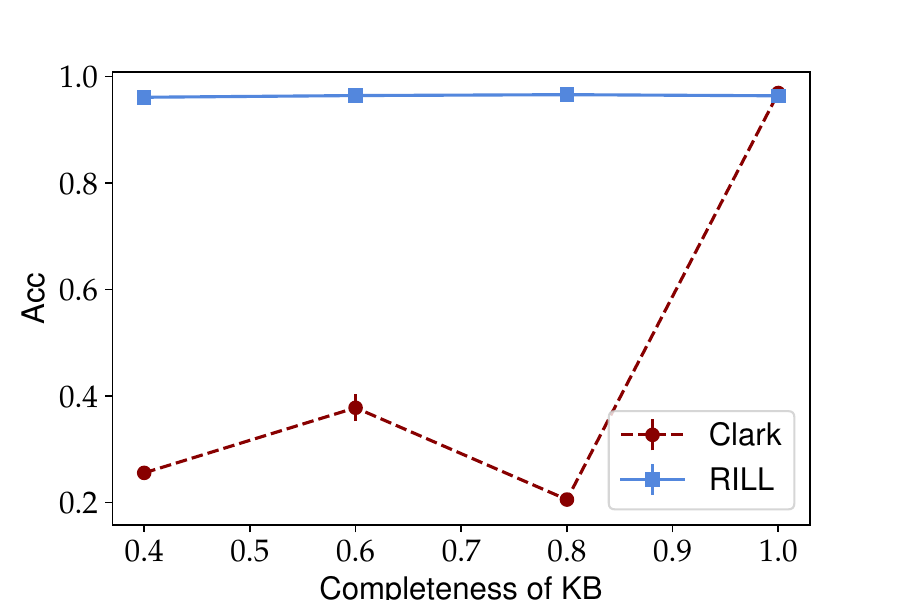}
        \captionsetup{font=small, labelfont=bf}
        \caption{Performance of Clark's Completion drops when the completeness of KB is decreasing.\label{app:fig:clark}}
    \end{figure}

    Surprisingly, explicitly applying Clark's completion in a NeSy system may not be helpful.
    There are two reasons behind this claim.

    First, a explicitly Clark's completion need to replace $\rightarrow$ to $\leftrightarrow$.
    For example:

    \begin{equation*}
        A\leftarrow A_1, A\leftarrow A_2,\cdots, A\leftarrow A_n,
    \end{equation*}

    will be replaced as $A\leftrightarrow (A_1\lor A_2\lor \cdots \lor A_n)$.

    However, the fuzzy operator is unsuitable for approximating a rule with many atoms~\citep{DBLP:journals/corr/abs-2108-11451}. An example is the n-ary Łukasiewicz strong disjunction $F_{\lor}(x_1,\cdots,x_n) = \min(1,x_1+\cdots+x_n)$.
    Although all $x_i$ can be very small, this approximation of disjunction will give a value near 1.
    Because Clark's completion will increase the number of atoms in the logical rule, the knowledge base may suffer from this problem after completion.

    Second, in a NeSy system, if the knowledge base is incomplete, replace $\rightarrow$ to $\leftrightarrow$
    will change the information in the knowledge base, which may induce wrong information.
    In logic programming, Clark's completion will not change the soundness of the system, while in the NeSy setting with a data-driven approach, it may not be promised.

    Here we adopt the same experimental setting of incomplete knowledge base case in \cref{sec:exp:add} to validate the performance of Clark's Completion.
    As depicted in \cref{app:fig:clark}, when the knowledge base becomes incomplete, completion of the knowledge base will not help improve the model's performance because it introduces wrong information.

    \section{Sensitivity Analysis}
    Both $\text{RILL}_{hinge}$ and $\text{RILL}_{l2+hinge}$ contain a hyper-parameter $\epsilon$ in their definition. In this section, we investigate the sensitivity of $\epsilon$ and its impact on the model's accuracy in the Add-MNIST experiment when the knowledge base incompleteness is 40\%. \cref{fig:sa} displays the relationship between $\epsilon$ and the model's accuracy.

    The results indicate that when the threshold $\epsilon$ decreases to a certain value (in \cref{fig:sa}, it is 0.001), the model's performance drops, suggesting the existence of a sensitive region that could be related to weak samples in addressing the problem of implication bias. When the threshold $\epsilon$ is above this value, the performance of RILL is stable but slightly decreasing. This decrease may be due to the increased number of weak samples, which results in the loss of some useful information.
    \begin{figure}[!htb]
        \centering
        \begin{subfigure}[t]{.45\linewidth}
            \centering
            \includegraphics*[width=0.8\linewidth]{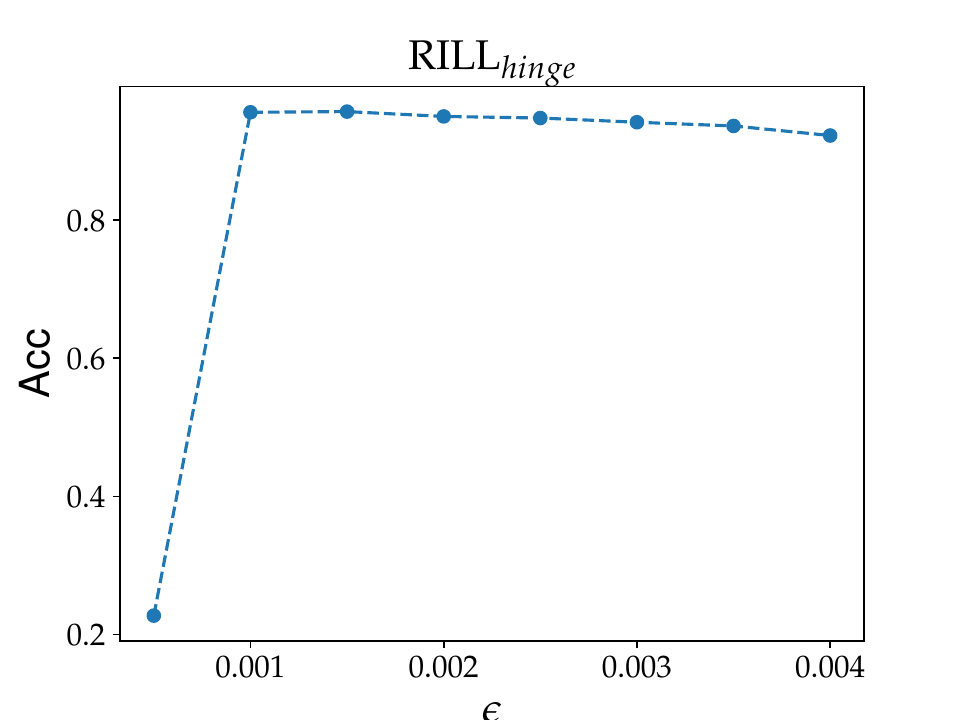}
            \caption{ \label{fig:sa_mnist_100_40_u}$\text{RILL}_{hinge}$}
        \end{subfigure}
        \quad
        \begin{subfigure}[t]{.45\linewidth}
            \centering
            \includegraphics*[width=0.8\linewidth]{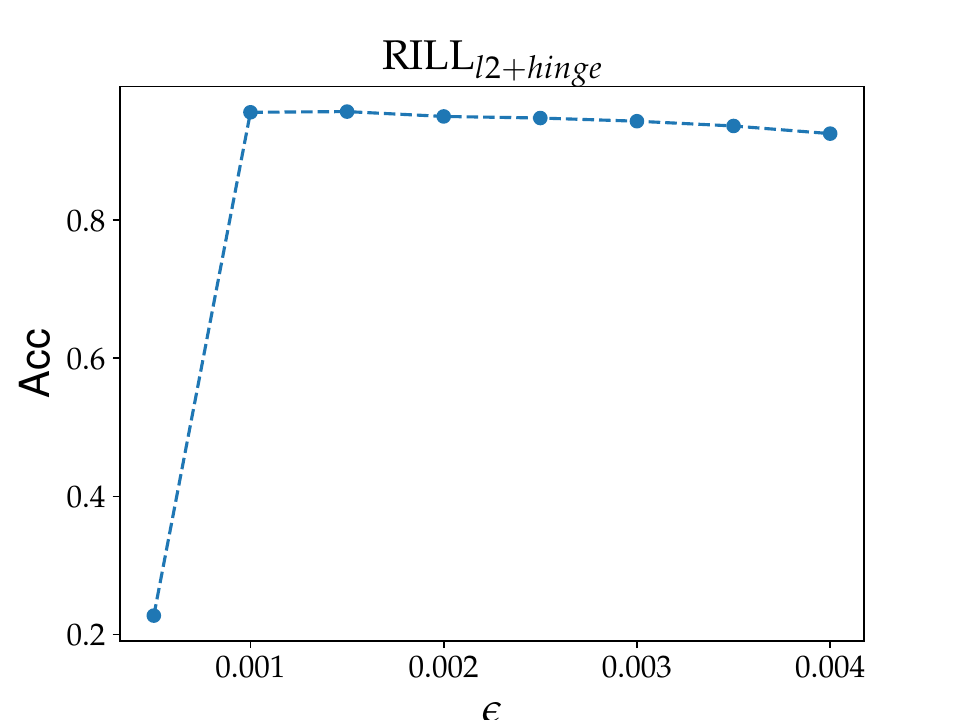}
            \caption{ \label{fig:sa_mnist_100_40_s}$\text{RILL}_{l2+hinge}$}
        \end{subfigure}
        \caption{Sensitivity Analysis of the hyper-parameter $\epsilon$.\label{fig:sa}}
    \end{figure}
    \section{Details of Experiments}
    In this section, we will provide more details of our experiments.
    \bmhead{Implementation of Semantic Loss}
    The formal definition of Semantic Loss requires the conversion of CNF (conjunctive normal form) into SDD (sentential decision diagram), which can become impractical when dealing with knowledge bases that use a large number of predicates and complex rules due to the heavy computational burden.
    To address this issue, we propose an alternative approach where each rule in the knowledge base is converted separately into wmc (weighted model counting) models and then combined at the end. Although explicit repetitive terms are reduced during the combining phase, detecting implicit repetitive terms is computationally expensive, and thus they remain unchanged. This approach enables us to implement Semantic Loss efficiently while still accounting for the complexity of the knowledge base.

    However, even with this optimization, the cost of using Semantic Loss is still much higher than that of using RILL, with a cost that is around three times higher. As a result, Semantic Loss may not be practical to use in many cases, particularly when dealing with large or complex knowledge bases.

    \bmhead{Details of Task 1}
    The backbone for both MNIST and FashionMNIST datasets is a three-layer Multilayer Perceptron (MLP) with a width of each layer being [256,512,10], and the activation function is Rectified Linear Unit (ReLU)~\citep{maas2013rectifier}. In contrast, the backbone for CIFAR10 is ResNet9~\citep{he2016deep}.
    The value of $\lambda$ for both Fuzzy and RILL logic loss is 0.7. For Semantic Loss, we sample $\lambda$ from $\{0.001,0.005,0.01,0.05,0.1,0.5\}$ and choose the best one, which is 0.5 in this experiment.
    The learning rate is set to 0.0001, with a decay rate of 0.7. The learning rate scheduler is set to StepLR, with a decay step of 60. The optimizer used is AdamW as default, and the weight decay rate is set to 5e-4.

    \bmhead{Details of Task 2}
    For this task, we select WideResNet-28-8~\citep{zagoruyko2016wide} as the backbone architecture with two classification heads, one for \textit{class} classification, and the other for \textit{super-class} classification.
    The value of $\lambda$ for both Fuzzy and RILL logic loss is 0.002, and for Semantic Loss, we still choose 0.5 as the default value.
    The learning rate is set to 0.005, with a decay rate of 0.9 and decay steps equal to 45. The learning rate scheduler is set to StepWithWarmUp, and the warm-up epoch is set as 5. The optimizer used is set as default, and momentum is set to 0.9.

\bmhead{Availability of data and material} All datasets are publicly available.
\bmhead{Code availability} Available at \href{https://git.nju.edu.cn/Alkane/clion.git}{https://git.nju.edu.cn/Alkane/clion.git}.
\bmhead{Author's contributions} H conceived the central idea of this paper and contributed to the writing and execution of the main experiments. D contributed to the enhancement of the RILL approach and the refinement of this paper. L provided valuable feedback, suggestions, and editing services for this manuscript.
\section*{Declarations}
\bmhead{Funding} Not applicable.
\bmhead{Conflicts of interest} Not applicable.
\bmhead{Ethics approval} Not applicable.
\bmhead{Consent to participate} Not applicable.
\bmhead{Consent for publication} Not applicable.

\end{appendices}
\bibliographystyle{plainnat}
{
\bibliography{strings,ref}% common bib file
}
\end{document}